%% file: main.tex
\documentclass[twoside]{article}

\usepackage[accepted]{aistats2026}

\usepackage{amsfonts}
\usepackage{amsthm}
\usepackage{algorithm}
\usepackage{algorithmic}
\usepackage{booktabs}

\usepackage{multirow}
\usepackage{natbib}
\usepackage{graphicx}     % for \includegraphics
\usepackage{subcaption}   % optional: for subcaptions
\usepackage{amsmath}
\usepackage{amssymb}
\usepackage{stfloats}
\usepackage{hyperref}
\usepackage{relsize}
\usepackage{placeins}
\usepackage{adjustbox}

\graphicspath{{img/}}     % look for images in the img/ directory

\newtheorem{theorem}{Theorem}

\begin{document}



\twocolumn[

\aistatstitle{TESLA: Taylor Expansion of Sinusoidal Learnable Activations}

\aistatsauthor{
Daehwa Ko \And
Jaehyeon Kim \And
Seunghyun Ham \And
Jay Hoon Jung
}

\aistatsaddress{
Korea Aerospace University, Goyang, Republic of Korea
}
]

\begin{abstract}
The parity problem—deciding whether the number of ones in a binary vector is odd or even—remains challenging for standard neural networks due to linear inseparability and the need for global interactions. We propose TESLA, an activation defined as a learnable combination of sine and cosine terms, enabling explicit control over polynomial degree and selective amplification of high-order components. Theoretically, we show that constraining TESLA’s coefficients yields Lipschitz/Rademacher complexity bounds and shapes the training dynamics to emphasize higher-frequency structure. Empirically, on parity with input length $n=32$, TESLA attains strong generalization with 100K training samples ($\approx 0.002\%$ of the $2^{32}$ input space) and remains robust under heavy corruption, retaining high accuracy with up to 30\% label noise. We also compare against periodic and frequency-based baselines (SIREN, SNAKE, and Fourier feature embeddings) on parity and Forrelation. Beyond synthetic structure, TESLA delivers comparable performance on ImageNet-100, indicating that activation-level degree control transfers to more general vision workloads. Code: \url{https://github.com/KAU-QuantumAILab/TESLA}
\end{abstract}

\begin{figure}[!t]
  \centering
  \includegraphics[width=7cm]{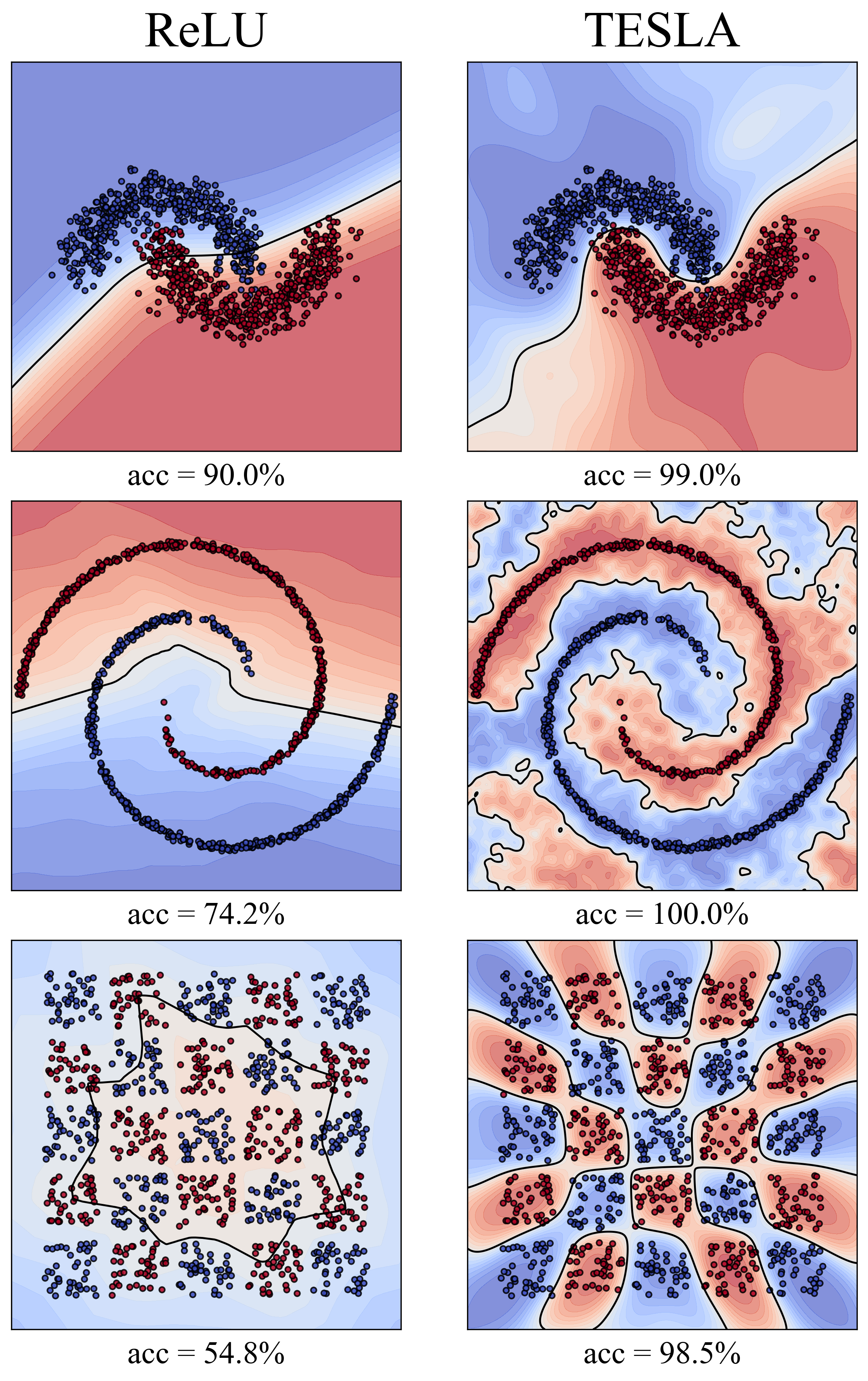} 
  \caption{Visualization of the decision boundaries for TESLA($K=8$) and baseline activations on a 2D feature space. The solid lines separate the predicted classes.}
  \label{fig:decision_boundary}
\end{figure}

\section{INTRODUCTION}
Modern neural networks largely build complex functions by stacking local nonlinearities such as ReLU \citep{ReLU} and GeLU \citep{GeLU} with linear maps. These designs excel at capturing local structure in images and language, but they can be inefficient for global or long-range interaction tasks that depend on high-order combinations of many input coordinates, global parity/phase, or other global symmetries \citep{parity_problem}. Empirically and theoretically, traditional activation functions and vanilla MLPs exhibit a spectral bias toward low-frequency (low-degree) components, so representing strong high-order or global structure often requires substantially greater depth, width, or very large weight norms \citep{rahaman2019spectral,xu2019frequency,tancik2020fourier,sitzmann2020siren}.

To enable direct and learnable control over the polynomial order at the activation level, we propose a simple parametric activation function based on a finite sine and cosine basis. This formulation allows the network to learn Fourier-like coefficients that explicitly control the effective polynomial order of activations. We refer to this activation as Taylor Expansion of Sinusoidal Learnable Activation (TESLA). Our main contributions are as follows:
\begin{itemize}
  \item We derive tight Lipschitz bounds, establish Rademacher complexity--based generalization guarantees, and characterize learning dynamics, showing that TESLA’s inductive bias naturally favors recovery of low- and mid-frequency structure.
  
  \item We present constructive approximation results for ridge- and interaction-type functions, together with complementary lower bounds showing that standard piecewise-linear activations require substantially more resources to match TESLA’s ability to represent oscillatory structure.
  \item On synthetic benchmarks, TESLA generalizes a 32-bit parity task with \textbf{100K training samples} and remains robust under label noise up to 30\%. It also improves performance on Forrelation and LPN tasks and remains usable in ImageNet-100-scale models while incurring negligible throughput overhead ($\approx 1\%$).
\end{itemize}

\section{RELATED WORK}
\label{sec:related}
Recent research has explored various ways to enrich neural network representations through activation design. Several works, such as PReLU \citep{P-relu} and APL \citep{APL}, introduce parametric nonlinear activations with a small number of learnable parameters to improve expressiveness and facilitate optimization. While these methods directly modify the pointwise nonlinearity, they do not provide explicit mechanisms for controlling polynomial or spectral degree. In contrast, approaches that learn per-edge or per-unit activations—such as KANs \citep{KAN2024} and spline-based methods \citep{Apicella2020, Scardapane2016}—offer substantially greater expressivity and can closely approximate complex mathematical or scientific functions. However, prior studies and surveys have noted that this increased flexibility can lead to overfitting or instability in low-data regimes unless it is carefully regulated through techniques such as regularization or parameter sharing.

Building on these insights, a complementary line of work employs sinusoidal and Fourier-based parameterizations to address the spectral limitations of standard activations and embeddings. These approaches aim to expand the spectral capacity of neural networks and reshape their training dynamics, motivating the development of periodic activations and Fourier feature mappings. Periodic activations (e.g., SIREN \citep{sitzmann2020siren}) and positional or Fourier feature encodings \citep{tancik2020fourier, rahimi2007random} mitigate spectral bias by introducing high-frequency components into the network—either by expanding its effective spectral support or by modifying the spectrum of the initial Neural Tangent Kernel (NTK, \citep{NTK}). The NTK is a theoretical construct that characterizes the training dynamics of infinitely wide neural networks, showing that such networks evolve like linear models governed by a fixed kernel, which in turn guarantees convergence to a global minimum.

In parallel, a substantial body of theoretical work has emerged to analyze spectral bias (or the frequency principle) and NTK dynamics, aiming to explain why standard networks prioritize low-frequency components and how modifying kernels or embeddings can alter this behavior \citep{rahaman2019spectral, xu2019frequency, NTK}. Classical approximation-theoretic results—such as Barron-type bounds \citep{barron1993universal, barron1994}, depth–width separation theorems \citep{eldan2016, telgarsky2016benefits}, and lower bounds for piecewise-linear approximations \citep{yarotsky2017, telgarsky2016benefits}—provide a rigorous framework for comparing different parameterizations. These results motivate our constructive ridge-approximation upper bounds and lower-bound sketches, which characterize when piecewise-linear activations incur high resource costs to emulate oscillatory global features \citep{barron1993universal, telgarsky2016benefits}.

\section{PROPOSED METHOD}
\label{sec:method}
\subsection{Activation Function Definition}
We define a parametric activation function, shared across all neurons within a layer, built from a finite sine and cosine basis:
\begin{equation}
\label{eq:Tesla}
\phi_K(z) \;=\; \sum_{k=1}^K\!\left(\frac{a_k}{k}\sin(kz)+\frac{b_k}{k}\cos(kz)\right),
\end{equation}

where \(K\in\mathbb{N}\) controls the maximum frequency, and \(\{a_k,b_k\}_{k=1}^K\) are learnable scalar coefficients. 
For a neuron receiving input \(z=v^\top x\), the neuron output is \(f(x)=\phi_K(v^\top x)\). We also define the coefficient budget:
\begin{equation}
\label{eq:A_K}
A_K \;:=\; \sum_{k=1}^K \big(|a_k| + |b_k|\big),
\end{equation}

which appears in stability and complexity bounds.

\subsection{Taylor (Maclaurin) Expansion and Polynomial Coefficients}
Expanding \(\phi_K\) via the Taylor series of sine and cosine gives:
\begin{align}
\label{eq:taylor_expansion}
\phi_K(z) =& \sum_{k=1}^{K} \frac{b_k}{k} \nonumber + \sum_{m=0}^\infty \left( \frac{(-1)^m}{(2m+1)!} \sum_{k=1}^K a_k k^{2m} \right) z^{2m+1} \nonumber \\
& + \sum_{m=1}^\infty \left( \frac{(-1)^m}{(2m)!} \sum_{k=1}^K b_k k^{2m-1} \right) z^{2m}.
\end{align}

Eq.~\eqref{eq:taylor_expansion} shows that each polynomial order is an explicit linear functional of the learned coefficients \(\{a_k,b_k\}\), giving direct control over degree components at the activation level.

\paragraph{Interpretation vs.\ Implementation.}
TESLA is implemented as the finite trigonometric expansion in Eq.~\eqref{eq:Tesla}; we do not evaluate a polynomial series during the forward pass. We use Eq.~\eqref{eq:taylor_expansion} as an interpretation that links learned trigonometric coefficients to effective polynomial degree and motivates the coefficient budget in our analysis.

\paragraph{Input Spectralization vs.\ Activation Spectralization.}
Input-side spectralization (e.g., Fourier features; \citealp{tancik2020fourier}) maps inputs to high-frequency coordinates and leaves mode selection to later layers. TESLA instead spectralizes the activation itself: coefficients $\{a_k,b_k\}$ directly shape Maclaurin-order terms, enabling selective amplification or attenuation of target orders. 

In practice, this design reduces the need for hand-crafted frequency mappings at the input stage and moves spectral control into a small set of learnable activation coefficients. As a result, practitioners can tune inductive bias with fewer architectural changes while preserving compatibility with standard training pipelines. This makes TESLA particularly attractive when transferring across tasks with different frequency characteristics.

\section{THEORETICAL ANALYSIS}
\label{sec:theory}
In this section, we provide a concise theoretical foundation for TESLA. We first bound derivatives and Lipschitz constants to show how the per-harmonic $\ell_1$ budget $A_K$ controls gradient magnitudes. We then present Rademacher-complexity generalization bounds, analyze NTK-style mode-wise learning dynamics, and derive a practical rule for selecting $K$ as a function of the effective order $m_{\mathrm{eff}}$, sample size $N$, and budget $A_K$.

\subsection{Derivative Bound and Stability}
\label{sec:deriv-bound}
Differentiation of Eq.~\eqref{eq:Tesla} gives:
\[
\phi_K'(z)\;=\;\sum_{k=1}^K \big(a_k\cos(kz)-b_k\sin(kz)\big).
\]
Since $|\sin(\cdot)|\le 1$ and $|\cos(\cdot)|\le 1$, each term is bounded by its amplitude,
$|a_k\cos(kz)-b_k\sin(kz)|\le \sqrt{a_k^2+b_k^2}$; therefore,
\begin{equation*}
\label{eq:phiK-prime-bound}
\|\phi_K'\|_\infty
\;\le\;
\sum_{k=1}^K \sqrt{a_k^2+b_k^2}
\;\le\;
\sum_{k=1}^K (|a_k|+|b_k|)
\;=\;A_K.
\end{equation*}

\paragraph{Composition Lipschitz Bounds.}
Let $h(x)=\phi_K(Wx)$, where $\phi_K$ is applied elementwise. By the chain rule,
\[
\operatorname{Lip}(h)\le \|\phi_K'\|_\infty\,\|W\|_{\mathrm{op}} \le A_K\,\|W\|_{\mathrm{op}}.
\]
Here, $\|\cdot\|_{\mathrm{op}}$ denotes the spectral norm, and $\|\cdot\|_\infty$ denotes the $L_\infty$ norm.

% Let a layer $h(x)=\phi_K(Wx)$, where $\phi_K$ is applied elementwise. Then,
% \[
% \operatorname{Lip}(h)\le \|\phi_K'\|_\infty\,\|W\|_{\mathrm{op}} \le A_K\,\|W\|_{\mathrm{op}}.
% \]

% For a scalar neuron $f(x)=\phi_K(v^\top x)$ with $\quad \|v\|_2 \le C.$, the chain rule yields

% \[
% \begin{aligned}
% \nabla f(x) &= \phi_K'(v^\top x)\,v,\\
% \operatorname{Lip}(f)
% &= \sup_x \|\nabla f(x)\|_2\le \|\phi_K'\|_\infty \cdot \|v\|_2 \\
% &\le A_K \cdot C.
% \end{aligned}
% \]

% For an elementwise layer $h(x)=\phi_K(Wx)$ with operator (spectral) norm $\|W\|_{\mathrm{op}}$,
% \[
% \operatorname{Lip}(h)\le \|\phi_K'\|_\infty\,\|W\|_{\mathrm{op}}\le A_K\,\|W\|_{\mathrm{op}}.
% \]

\paragraph{Network-level Lipschitz and Degree Control.}
For an $L$-layer network
\[
f(x)=W_L \phi_K(\cdots \phi_K(W_1 x)), 
\qquad \|\phi_K'\|_\infty \le A_K,
\]
the Lipschitz constant satisfies
\[
\operatorname{Lip}(f)
\le
\left(\prod_{\ell=1}^L \|W_\ell\|_{\mathrm{op}}\right) A_K^{L-1}.
\]

Beyond Lipschitz control, the same composition also governs the growth of polynomial complexity: under the Maclaurin-coefficient view induced by Eq.~\eqref{eq:taylor_expansion}, composing $L-1$ TESLA layers activates terms up to order $O(K^{L-1})$.
Moreover, the $\ell_1$ norm of the corresponding coefficient vector is bounded as
\[
\|\mathrm{coef}(f)\|_1
\le
A_K^{L-1}
\prod_{\ell=1}^L \|W_\ell\|_{1\to 1}.
\]

Thus, a single per-layer budget $A_K$ stabilizes gradients, while $K$ controls the accessible frequency range. In particular, for an $L_{\mathrm{loss}}$-Lipschitz loss $\mathcal L$,
\[
\|\nabla_{W_\ell} \mathcal L(f(x),y)\|_F
\le
L_{\mathrm{loss}} \|x\|_2
\left(\prod_{j=1}^L \|W_j\|_{\mathrm{op}}\right)
A_K^{L-1}.
\]

This highlights the trade-off: increasing $K$ expands representable high frequencies, while controlling $A_K$ keeps optimization stable.

% \paragraph{Network-level Lipschitz and Degree Control.}
% For an $L$-layer network
% \[
% f(x)=W_L\phi_K(\cdots \phi_K(W_1x)),
% \qquad
% \|\phi'_K\|_\infty\le A_K,
% \]
% the Lipschitz constant satisfies
% \[
% \operatorname{Lip}(f)\le\big(\prod_{\ell=1}^{L}\|W_\ell\|_{\mathrm{op}}\big)\,A_K^{\,L-1}.
% \]
% Beyond Lipschitz control, the same composition also governs the growth of polynomial complexity in the Maclaurin-coefficient view.
% In the Maclaurin-coefficient view induced by Eq.~\eqref{eq:taylor_expansion},
% composing $L-1$ TESLA layers can activate terms up to order $O(K^{L-1})$.
% Beyond Lipschitz control, the same composition also governs the growth of polynomial complexity. In particular, under the Maclaurin-coefficient view induced by Eq.~\eqref{eq:taylor_expansion}, composing $L-1$ TESLA layers activates terms up to order $O(K^{L-1})$.
% Under the same $\ell_1$ budget, a corresponding truncated coefficient mass is bounded by
% \[\|\mathrm{coef}(f)\|_1\le A_K^{\,L-1}\,\prod_{\ell=1}^{L}\|W_\ell\|_{1\to1}.\]
% Thus a single per-layer budget $A_K$ stabilizes gradients while $K$ controls the accessible frequency range.
% In particular, for an $L_{\mathrm{loss}}$-Lipschitz loss $\mathcal{L}$,
% \[\|\nabla_{W_\ell}\mathcal{L}(f(x),y)\|_F \le L_{\mathrm{loss}}\,\|x\|_2\,\big(\prod_{j=1}^{L}\|W_j\|_{\mathrm{op}}\big)\,A_K^{\,L-1}.\]
% This highlights the trade-off: increasing $K$ expands representable high frequencies, while controlling $A_K$ keeps optimization stable.

\subsection{Rademacher Complexity and Generalization}

\begin{theorem}\label{thm:rademacher_main}
Assume $\|x_i\|_2\le R$ for all $i=1,\dots,N$. Define
\[
\mathcal{F}_K=\{\,x\mapsto \phi_K(v^\top x)\;|\;\|v\|_2\le W,\; A_K\le A\,\},
\]
and
\[
B_0(A):=\sup_{\phi_K:\,A_K\le A}|\phi_K(0)|.
\]
Then, for any sample of size $N$, the empirical Rademacher complexity satisfies
\[
\hat{\mathfrak{R}}_N(\mathcal{F}_K) \le \frac{AWR}{\sqrt{N}} + \frac{B_0(A)}{\sqrt{N}} = \frac{AWR + B_0(A)}{\sqrt{N}}.
\]
Consequently, for any $L_{\mathrm{loss}}$-Lipschitz loss $\mathcal{L}$, the generalization gap scales as
$O\!\left(L_{\mathrm{loss}}(AWR+B_0(A))/\sqrt{N}\right)$.
\end{theorem}

The bound recovers the familiar linear-class rate $\Theta(WR/\sqrt{N})$ up to the multiplicative factor $A$ due to the activation Lipschitz constant.
In deep architectures, the same effect accumulates multiplicatively via layer-wise operator norms and activation budgets.

\begin{proof}
Let
\[
\mathcal{F}_{\mathrm{lin}}:=\{x\mapsto v^\top x:\|v\|_2\le W\}.
\]
By a standard argument,
\[
\hat{\mathfrak{R}}_N(\mathcal{F}_{\mathrm{lin}})
= \frac{1}{N}\,\mathbb{E}\Big[\sup_{\|v\|_2\le W}\sum_{i=1}^N \sigma_i v^\top x_i\Big]
\le \frac{W R}{\sqrt{N}}.
\]
For any $\phi_K$ with $A_K\le A$, write $\phi_K(t)=\tilde\phi_K(t)+\phi_K(0)$ with $\tilde\phi_K(0)=0$ and $\operatorname{Lip}(\tilde\phi_K)\le A$.
Also, since $\phi_K(0)=\sum_{k=1}^K \frac{b_k}{k}$ and $A_K\le A$, we have $B_0(A)\le A$.
By the Ledoux--Talagrand contraction lemma~\citep{ledoux_talagrand},
\[
\hat{\mathfrak{R}}_N\big(\{\tilde\phi_K(v^\top x)\}\big)
\le A\,\hat{\mathfrak{R}}_N(\mathcal{F}_{\mathrm{lin}}) \le \frac{AWR}{\sqrt{N}}.
\]
For the constants,
\begin{align*}
\sup_{|c|\le B_0(A)}\frac{1}{N}\,\mathbb{E}\Big|\sum_{i=1}^N \sigma_i c\Big|
&= \frac{B_0(A)}{N}\,\mathbb{E}\Big|\sum_{i=1}^N \sigma_i\Big| \\
&\le \frac{B_0(A)}{N}\,\sqrt{\mathbb{E}\Big(\sum_{i=1}^N \sigma_i\Big)^2} \\
&= \frac{B_0(A)}{\sqrt{N}}.
\end{align*}
Combine the two parts to get the stated bound.
\end{proof}

\begin{figure}[t]
  \centering
  \includegraphics[width=0.95\linewidth]{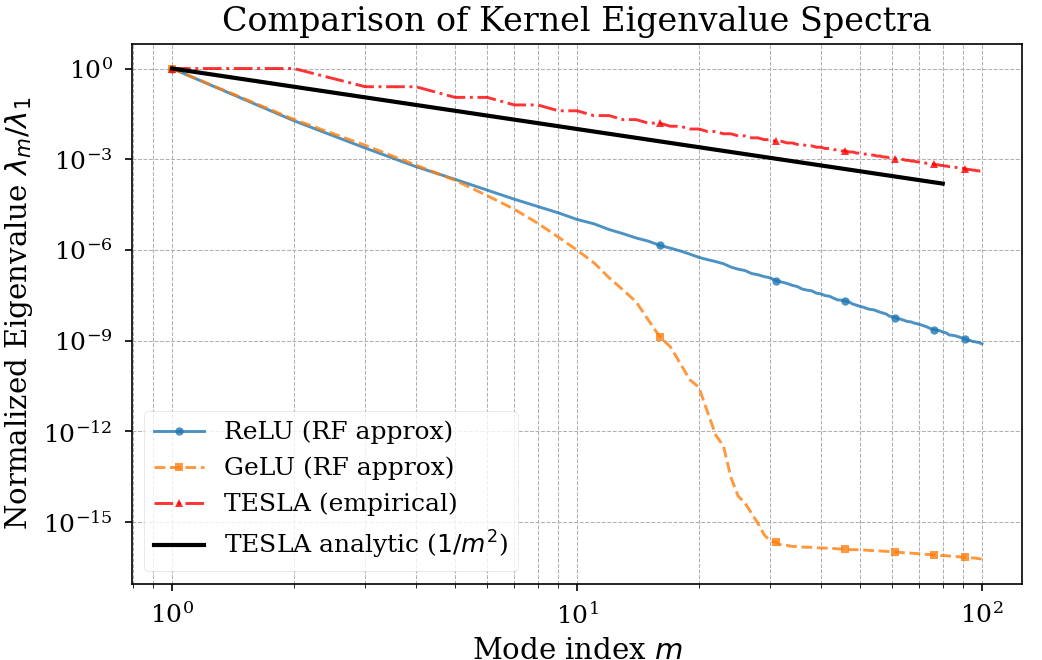}
  \caption{Mode-wise empirical spectra: normalized eigenvalues $\tilde\lambda_m=\lambda_m/\lambda_1$ vs.\ mode index $m$. Lines: TESLA analytic ($\tfrac{1}{2m^2}$), TESLA empirical, RF-ReLU, RF-GeLU. 
  RF denotes finite random-feature approximations of the infinite-width ReLU/GeLU kernels.}
  \label{fig:mode-spectra}
\end{figure}

\subsection{Mode-Wise Learning Dynamics}
\label{sec:modewise}

We analyze learning dynamics on the circle $\mathbb{T}=[0,2\pi)$ in the Fourier basis, where a \emph{mode} refers to a trigonometric component (e.g., $\sin(mt)$ or $\cos(mt)$) with frequency index $m$. We equip $\mathbb{T}$ with the inner product
\[
\langle f,g\rangle=\frac{1}{2\pi}\int_0^{2\pi} f(t)g(t)\,dt.
\]
Under squared loss, the linearized gradient flow around initialization is
\[
\partial_{\tau}(f_\tau-f^\star)=-\eta K_\phi(f_\tau-f^\star),
\]
where $\eta>0$ is the learning-rate scale and $K_\phi$ is the kernel operator
\begin{align*}
K_\phi(t,t') &= \nabla_\theta f_\theta(t)^\top \nabla_\theta f_\theta(t'),\\
(K_\phi g)(t)&=\int_0^{2\pi}K_\phi(t,t')g(t')\,\frac{dt'}{2\pi}.
\end{align*}

With TESLA coefficients $\{a_k,b_k\}_{k=1}^K$, the Jacobian features satisfy
\[
\frac{\partial f_\theta}{\partial a_k}(t)=\frac{1}{k}\sin(kt),\qquad
\frac{\partial f_\theta}{\partial b_k}(t)=\frac{1}{k}\cos(kt),
\]
so $K_\phi$ is translation-invariant with kernel
\[
K_\phi(t,t')=\sum_{k=1}^K \frac{1}{k^2}\cos\!\big(k(t-t')\big).
\]
The orthonormal Fourier basis $\psi_{m,s}(t)=\sqrt{2}\sin(mt)$ and $\psi_{m,c}(t)=\sqrt{2}\cos(mt)$ diagonalizes $K_\phi$:

\[
\begin{gathered}
K_\phi[\psi_{m,a}]=\lambda_m^{(\phi)}\psi_{m,a},\\
\lambda_m^{(\phi)}=\tfrac{1}{2m^2},\quad a\in\{s,c\},\ 1\le m\le K.
\end{gathered}
\]

and $\lambda_m^{(\phi)}=0$ for $m>K$. Writing the modal coefficients
\[
e_{m,a}(\tau)=\langle f_\tau-f^\star,\psi_{m,a}\rangle,
\]
the dynamics decouple as
\[
\begin{gathered}
e_{m,a}(\tau)=e_{m,a}(0)\exp\{-\eta\lambda_m^{(\phi)}\tau\},\\
\|f_\tau-f^\star\|_{L^2}^2=\sum_{m,a} e_{m,a}(\tau)^2.
\end{gathered}
\]
Thus, larger $\lambda_m^{(\phi)}$ yields faster decay of the corresponding mode. Analytically, TESLA gives $\lambda_m^{(\phi)}\propto m^{-2}$; empirically (Fig.~\ref{fig:mode-spectra}), its spectrum is flatter than RF proxies, allocating relatively larger eigenvalues to medium--high modes and accelerating recovery of higher-order structure.

\begin{figure*}[!t]
 \centering
 \includegraphics[width=14cm]{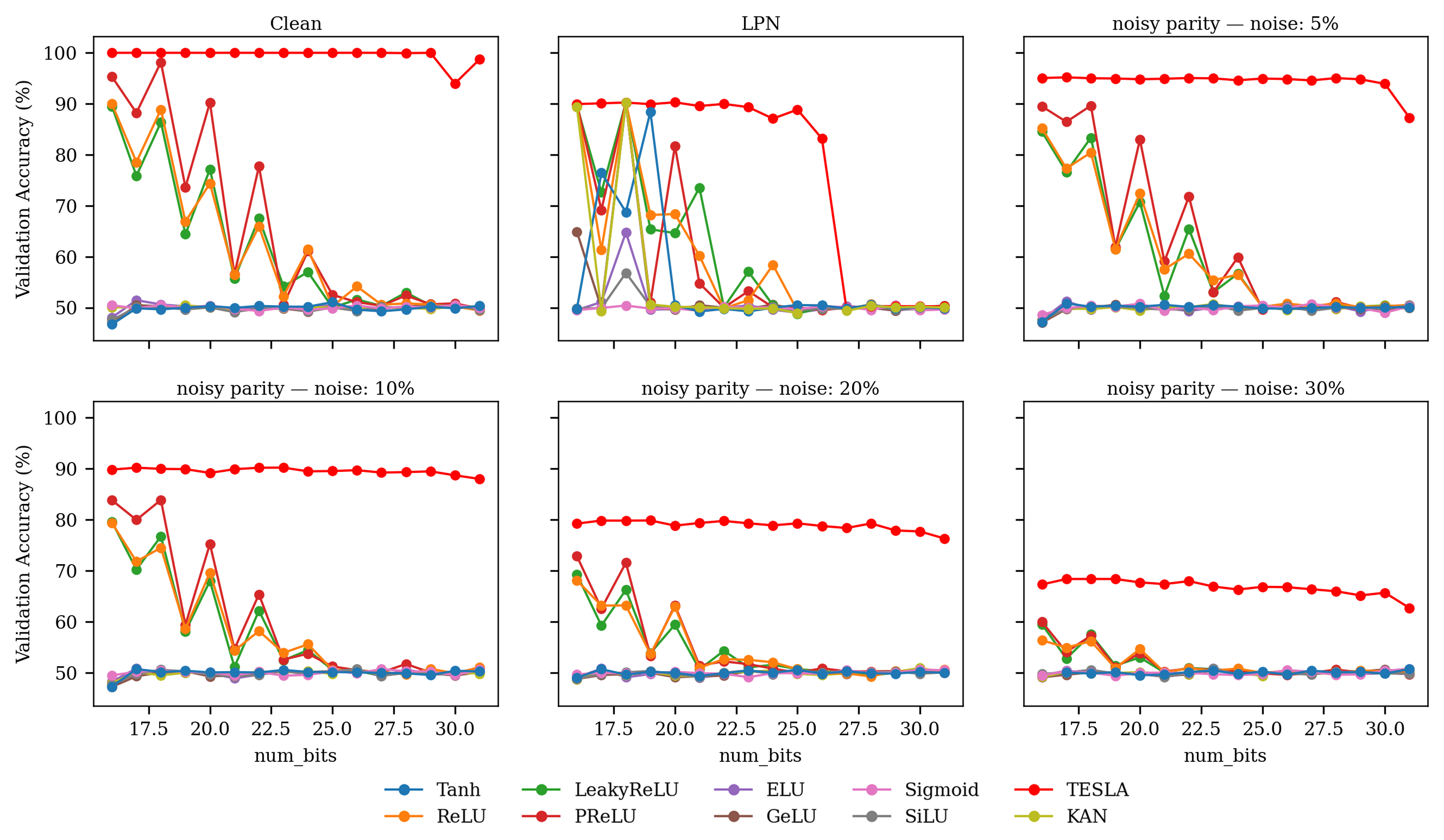}
 \caption{Validation accuracy (\%) for the parity task. Each panel shows a different noise condition: Clean, LPN (noise level = 10\%) and noisy-parity with $5\%$, $10\%$, $20\%$, $30\%$ label noise. 
 While most baseline activations remain near chance ($\approx 50\%$), TESLA maintains substantially higher accuracy across bit-lengths and noise levels.}
 \label{fig:main_acc_parity}
\end{figure*}

\subsection{Theory-Guided Choice of the Harmonic Count \texorpdfstring{$K$}{K}}
\label{sec:choose-K}
We give a compact, operational rule for how the harmonic count $K$ should scale with task complexity and sample size by balancing approximation and estimation.
Throughout this subsection, $N$ denotes the training sample size and $f^\star$ the target.
For approximation statements, we consider Boolean targets of the form $f^\star:\{0,1\}^d\to\{\pm1\}$, while using a continuous periodic proxy $f^\star:\mathbb T\to\mathbb R$ to interpret mode-wise NTK behavior.
Constants $c,C,C_0,\dots$ are absolute and may change value between displays.

\begin{figure*}[t]
  \centering
  \includegraphics[width=17cm]{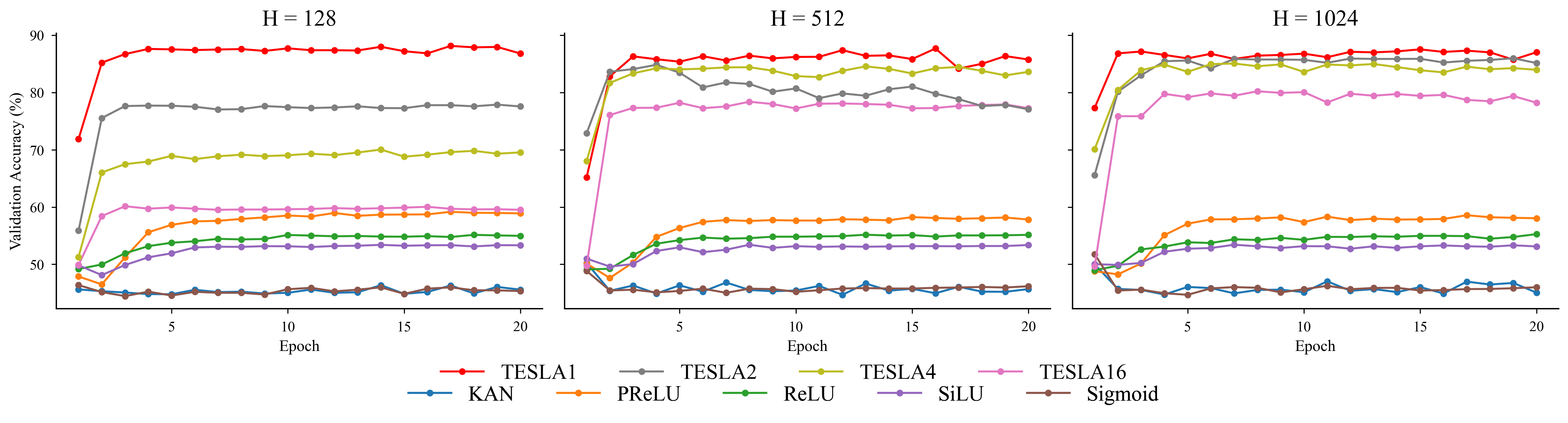}
  \caption{Validation accuracy (\%) for the Forrelation task with $d=12$ using a two-layer MLP. Each panel shows a different hidden size $H \in \{128, 512, 1024 \}$. Curves compare TESLA with $K \in \{1, 2, 4, 16\}$ against standard activations. 
  TESLA converges faster and achieves higher accuracy on this global-correlation task.}
  \label{fig:main_acc_forrelation}
\end{figure*}

\paragraph{Task Complexity via Effective Order.}
Write the Walsh expansion \citep{Walsh}
\[
f^\star(x)=\sum_{T\subseteq[d]}\hat f_T\,\chi_T(x). 
\]
For $\varepsilon\in(0,\tfrac12)$, define the effective interaction order
\[
m_{\mathrm{eff}}
:=\inf\Big\{m\in\{0,\dots,d\}\ :\ \sum_{|T|\le m}\hat f_T^2\ge 1-\varepsilon\Big\}.
\]
$m_{\mathrm{eff}}$ is the interaction order needed to capture most of the target energy (e.g., parity of order $m$ has $m_{\mathrm{eff}}=m$).

\paragraph{Approximation Behavior of $K$-harmonic Activations.}
Using Eq.~\eqref{eq:Tesla}, let the atoms $\{\sigma_k\}$ be normalized by $\|\sigma_k\|_\infty\le 1$ with $1/k$ scaling. Assume the best-in-class $L^2$ approximation error over $\mathcal{F}_K$ satisfies
\begin{align}
\label{eq:approx-decay-main}
\inf_{f\in\mathcal{F}_K}\ \mathbb{E}\!\left[(f(x)-f^\star(x))^2\right] \ \le\ B\!\Big(\frac{K}{m_{\mathrm{eff}}}\Big), \nonumber \\
 B(z)\in\{c_1 e^{-\gamma z},\ c_1 z^{-p}\},
\end{align}
for some $c_1,\gamma,p>0$; i.e., once $K=\Theta(m_{\mathrm{eff}})$, the approximation error decays monotonically.

\paragraph{Estimation Term for $\ell_1$-Budgeted Mixtures.}
Let $\hat{\mathfrak{R}}_N$ denote empirical Rademacher complexity on $N$ samples.
By the $\ell_1$-budgeted mixture bound (budget $A_K$), 
\begin{equation}
\label{eq:rad-mix-main}
\hat{\mathfrak{R}}_N(\mathcal{F}_K)\ \le\ C_0\,\frac{A_K}{\sqrt{N}}\sqrt{\log(2K)}.
\end{equation}
Consequently, for any $1$-Lipschitz surrogate loss $\mathcal{L}$,
with probability at least $1-\delta$,
\begin{align}
\label{eq:gen-main}
\sup_{f\in\mathcal{F}_K} \Big|\,\mathbb{E}[\mathcal{L}(f(x),y)]-\tfrac{1}{N}\sum_{i=1}^N \mathcal{L}(f(x_i),y_i)\,\Big| \nonumber \\
\le\ C_1\,\frac{A_K}{\sqrt{N}}\sqrt{\log(2K)}\ + C_2\sqrt{\tfrac{\log(1/\delta)}{N}}.
\end{align}

\paragraph{Excess-risk Bound and Optimal Scale.}
Let $\hat f_K\in\mathcal{F}_K$ minimize empirical surrogate risk.
Combining Eq.~\eqref{eq:approx-decay-main}–\eqref{eq:gen-main} yields
\begin{equation}
\label{eq:excess-main}
\mathcal{E}(\hat f_K)
\lesssim
\underbrace{B\!\Big(\tfrac{K}{m_{\mathrm{eff}}}\Big)}_{\text{approximation}}
+
\underbrace{C\,\frac{A_K}{\sqrt{N}}\sqrt{\log(2K)}}_{\text{estimation}}
+
C\sqrt{\tfrac{\log(1/\delta)}{N}}.
\end{equation}
Balancing the first two terms gives
\[
K^\star\ =\ \Theta\!\big(m_{\mathrm{eff}}\big)\cdot \kappa(N,A_K),
\]
where $\kappa$ varies slowly with $N$ and $A_K$:
for $B(z)=c_1 e^{-\gamma z}$, $\kappa=\Theta(\log N)$; 
for $B(z)=c_1 z^{-p}$, $\kappa=\Theta\!\big(N^{1/(2p)}(\log N)^{1/(2p)}\big)$ up to constants.
Intuitively, taking $K$ much larger than $m_{\mathrm{eff}}$ dilutes the per-harmonic budget (average amplitude $\sim A_K/K$) and increases estimation error.

In $d$-bit Parity, $m_{\mathrm{eff}}=d$, hence the theory predicts $K^\star=\Theta(d)$ up to slowly varying factors. 
Empirically (Sec.~\ref{sec:experiments}), we find a stable range $\kappa\in[0.2,0.4]$ at $N=10^5$, placing the optimum near $K\approx d/4$.
This aligns with the NTK view: the analytic TESLA kernel on $\mathbb{T}$ has eigenvalues $\lambda_m^{(\phi)}=\tfrac{1}{2m^2}$, but finite-feature and finite-sample effects flatten the empirical spectrum and emphasize medium--high modes. Thus, once $K\gtrsim m_{\mathrm{eff}}$, marginal approximation gains are offset by the estimation term in Eq.~\eqref{eq:excess-main}.

\section{EXPERIMENTS}
\label{sec:experiments}
\subsection{Comparisons with Periodic and Frequency-based Baselines}
\label{sec:periodic-baselines}
To provide a balanced comparison with periodic and Fourier-style activations, we include SIREN, SNAKE, and a Fourier-feature embedding baseline in both parity and Forrelation experiments, using matched architectures and training protocols unless otherwise noted.
For the Fourier-feature embedding baseline, we use a 64-dimensional input mapping with $\sin(2^k\pi x)$ and $\cos(2^k\pi x)$ for $k=0{:}31$, followed by the same MLP architecture with ReLU.

\subsection{Continuous-domain Evaluation}
\label{sec:continuous-domain}
Sinusoidal activations are often effective in continuous-domain learning problems. We therefore evaluate TESLA on three representative settings: (i) physics-informed neural networks (PINNs) for PDE solving, (ii) implicit neural representations (INRs) for coordinate-based image reconstruction, and (iii) mixed-frequency regression for explicit high-frequency signal modeling.

\paragraph{Physics-informed neural networks (PINNs).}
We train a PINN for the 1D viscous Burgers' equation. The total loss is
\begin{equation}
  \mathcal{L}_\text{total}
  = \mathcal{L}_\text{PDE}
  + 100\,\mathcal{L}_\text{IC}
  + 100\,\mathcal{L}_\text{BC},
\end{equation}
where $\mathcal{L}_\text{PDE}$, $\mathcal{L}_\text{IC}$, and $\mathcal{L}_\text{BC}$ are mean-squared errors of the PDE residual, initial condition, and boundary conditions, respectively.

\begin{table}[h]
\centering
\small
\setlength{\tabcolsep}{4pt}
\renewcommand{\arraystretch}{1.05}
\caption{1D Burgers' equation PINN: final losses after 20{,}000 epochs (MSE).}
\begin{adjustbox}{max width=\linewidth}
\begin{tabular}{lcccc}
\toprule
Activation &
$\mathcal{L}_\text{total}$ &
$\mathcal{L}_\text{PDE}$ &
$\mathcal{L}_\text{IC}$ &
$\mathcal{L}_\text{BC}$ \\
\midrule
ReLU  & $7.13\times 10^{-1}$ & $5.52\times 10^{-1}$ & $1.31\times 10^{-3}$ & $3.03\times 10^{-4}$ \\
Tanh  & $1.19\times 10^{-2}$ & $3.81\times 10^{-3}$ & $2.93\times 10^{-6}$ & $7.77\times 10^{-5}$ \\
SIREN & $1.04\times 10^{-2}$ & $6.00\times 10^{-3}$ & $1.06\times 10^{-5}$ & $3.38\times 10^{-5}$ \\
TESLA & $\mathbf{1.21\times 10^{-3}}$ & $\mathbf{5.01\times 10^{-4}}$ &
$\mathbf{1.11\times 10^{-6}}$ & $\mathbf{6.00\times 10^{-6}}$ \\
\bottomrule
\end{tabular}
\end{adjustbox}
\label{tab:burgers_pinn_final}
\end{table}

As shown in Table~\ref{tab:burgers_pinn_final}, TESLA achieves the lowest final loss on the Burgers' equation PINN among the compared activations.

\paragraph{Implicit neural representations (INRs).}
We evaluate INRs on Kodak24 and DIV2K, where a 6-layer MLP (512 hidden units) is trained per image for 3000 epochs under an identical training budget.

\begin{table}[h]
\centering
\small
\setlength{\tabcolsep}{4pt}
\renewcommand{\arraystretch}{1.05}
\caption{Implicit neural representations on Kodak24 and DIV2K. We report reconstruction quality (PSNR/SSIM) and the number of epochs needed to reach 25 dB (Epochs@25dB).}
\begin{adjustbox}{max width=\linewidth}
\begin{tabular}{l l c c c}
\toprule
Dataset & Activation & PSNR $\uparrow$ & SSIM $\uparrow$ & Epochs@25dB $\downarrow$ \\
\midrule
\multirow{4}{*}{Kodak24}
& TESLA & \textbf{27.72} & \textbf{0.9671} & \textbf{1595.8} \\
& SIREN & 25.33 & 0.9404 & 1887.5 \\
& ReLU  & 21.22 & 0.8389 & 3000.0 \\
& Tanh  & 20.17 & 0.7818 & 3000.0 \\
\midrule
\multirow{4}{*}{DIV2K}
& TESLA & \textbf{24.73} & \textbf{0.9522} & \textbf{2485.0} \\
& SIREN & 22.77 & 0.9301 & 2695.0 \\
& ReLU  & 18.21 & 0.8030 & 3000.0 \\
& Tanh  & 17.94 & 0.7819 & 3000.0 \\
\bottomrule
\end{tabular}
\end{adjustbox}
\label{tab:inr_activation}
\end{table}

\paragraph{Mixed-frequency regression.}
We regress a target signal composed of mixed sine/cosine components from 3.29~Hz to 79.90~Hz. Table~\ref{tab:combined_results} reports the dominant frequencies in the target (left) and train/test MSE of different activations (right).

\begin{table}[h]
    \centering
    \caption{Left: top 5 dominant frequencies of the target synthetic function. Right: train/test MSE comparison across activation functions for mixed-frequency signal modeling.}
    \label{tab:combined_results}
    \begin{minipage}[t]{0.55\linewidth}
        \centering
        \begin{adjustbox}{max width=\linewidth}
        \begin{tabular}{cccc}
            \toprule
            \textbf{Rank} & \textbf{Freq (Hz)} & \textbf{Amp} & \textbf{Phase} \\
            \midrule
            1 & 79.19 & 0.98 & 1.79 \\
            2 & 79.23 & 0.96 & 4.41 \\
            3 & 9.78  & 0.95 & 5.97 \\
            4 & 18.02 & 0.92 & 1.42 \\
            5 & 3.29  & 0.89 & 4.72 \\
            \bottomrule
        \end{tabular}
        \end{adjustbox}
    \end{minipage}
    \hfill
    \begin{minipage}[t]{0.43\linewidth}
        \centering
        \begin{adjustbox}{max width=\linewidth}
        \begin{tabular}{lcc}
            \toprule
            \textbf{Model} & \textbf{Train} $\downarrow$ & \textbf{Test} $\downarrow$ \\
            \midrule
            ReLU  & 0.0682 & 0.0681 \\
            GeLU  & 0.0752 & 0.0756 \\
            SiLU  & 0.0753 & 0.0756 \\
            SIREN & 0.0236 & 0.0248 \\
            TESLA & \textbf{0.0097} & \textbf{0.0100} \\
            \bottomrule
        \end{tabular}
        \end{adjustbox}
    \end{minipage}
\end{table}
In this mixed-frequency setting, TESLA is the only method with test MSE below 0.011.

\subsection{Parity with Label Noise: Task, Metrics, and Interpretation}
\label{sec:parity-noise}

\paragraph{Task.}
We evaluate on Parity and global-statistics problems where the target depends on a high-order global interaction of the input bits. 
Let $x\in\{0,1\}^d$ and $S\subseteq[d]$ with $|S|=m$. The clean label is
\[
y_{\mathrm{clean}}=\bigoplus_{i\in S} x_i \in\{0,1\},
\]
i.e., the parity of the selected bits.\footnote{We use a fixed subset $S$ per run; $d$ (number of bits) controls task difficulty.} 
This benchmark is intentionally adversarial to local, piecewise-linear activations, as the Bayes-optimal classifier is a global rule with a single nonzero Fourier coefficient at frequency $S$.

\paragraph{Data Generation, Splits, and Noise Settings.}
For each run, we sample i.i.d. inputs uniformly from $\{0,1\}^d$ and assign splits deterministically via a fixed hash partition $h(x)\bmod 3$, ensuring disjoint and reproducible train/validation/test sets (i.e., no leakage). We then sample without replacement within each split, using 100{,}000 training samples and 20{,}000 validation samples for all $d$. We evaluate three settings: (i) clean parity, (ii) noisy parity with independent label flips $p \in\{0.05,0.10,0.20,0.30\}$, and (iii) LPN; in the noisy setting, train and test use the same flip rate $p$.

\paragraph{LPN.}
We also evaluate the classical LPN (Learning Parity with Noise) variant, where the relevant subset is unknown and must be inferred from data. 
Fix a hidden secret vector $s\in\{0,1\}^d$ and draw query vectors $a\sim\mathrm{Unif}(\{0,1\}^d)$ i.i.d. The observed label is
\[
y \;=\; \langle a, s\rangle \bmod 2 \;\oplus\; e,\quad e\sim\mathrm{Bernoulli}(\eta).
\]
Compared to the noisy parity setting in Sec.~\ref{sec:parity-noise}---where the subset $S$ is fixed per run and implicitly known to the data generator---LPN additionally hides $S$ (equivalently $s$), making both identification and optimization more challenging at larger $d$.

\paragraph{Metrics and Bayes Limit.}
With symmetric label noise (flip rate $p<\tfrac12$) on the test set, perfect recovery of the underlying rule cannot yield $100\%$ measured accuracy. The Bayes-optimal expected test accuracy under symmetric, input-independent label flips at rate $p$ is $A_{\max}(p)=1-p$.

\subsection{Forrelation}
Forrelation \citep{forrelation} is a decision problem on pairs of Boolean functions $f, g: \{0,1\}^n \rightarrow \{ \pm 1 \}$. Let $M=2^n$. The goal is to decide whether the correlation
\begin{equation*}
    \Phi_{f,g} = \frac{1}{\sqrt{M}} \langle Hf, g \rangle
\end{equation*}
is high or low under a standard promise. Here $H$ is the Hadamard matrix and $\langle \cdot, \cdot \rangle$ denotes the inner product. We treat $\Phi_{f,g} \ge 0.6$ as positive and $|\Phi_{f,g}| \le 0.01$ as negative. The property of being forrelated is global. A classifier must integrate information spread across the entire pair of truth tables. We use a two-layer MLP that takes as input the concatenation of the truth tables of $f$ and $g$ with length $2^{n+1}$ and outputs a single logit for binary prediction. We train and evaluate on a synthetic dataset with $n=12$, consisting of 10{,}000 training examples and 10{,}000 test examples, where labels are assigned as $y=1$ if $\Phi_{f,g} \ge 0.6$ and $y=0$ if $| \Phi_{f,g}| \le 0.01$, excluding samples with intermediate values.

\subsection{ImageNet-100 Classification}
To empirically validate the performance and efficiency of our proposed activation function, TESLA, we conducted experiments on the ImageNet-100 dataset \citep{ImageNet}, a standard benchmark for image classification. To demonstrate TESLA's general applicability, our evaluation includes both Transformer-based models (ViT \citep{Vit}, MLP-Mixer \citep{MLP-mixer}) and CNN-based models (ResNet \citep{Resnet}, MobileNetV3 \citep{MobileNetv3}). We compare TESLA's performance against conventional activation functions, namely ReLU \citep{ReLU}, GeLU \citep{GeLU}, and SiLU \citep{Silu}. TESLA uses $K=2$ for ViT-T/16, ResNet-18, and MobileNetV3, and $K=6$ for MLP-Mixer. The evaluation is based on a comprehensive analysis of not only the classification performance (Top-1 Accuracy) but also key computational efficiency metrics, including FLOPs, throughput (imgs/s), and the number of parameters. Table~\ref{tab:act_compute_acc} shows these results. The top-1 accuracy is reported as the mean and standard deviation over three seeds and both the measurement of FLOPs and throughput measurement are conducted on a single GPU.

\begin{table}[t]
\centering
\caption{ImageNet-100 results across architectures comparing TESLA with common activations. Top-1 accuracy is mean $\pm$ std (3 seeds). Bold indicates best per architecture. FLOPs reported in GigaFLOPs, throughput (imgs/s) measured on a single GPU, and parameters (Params) in millions.}
\resizebox{8.5cm}{!}{\begin{tabular}{lcccccc}
\toprule
Model & Act & Top-1 (\%) & FLOPs & imgs/s & Params \\
\midrule
\multirow{4}{*}{ViT-T/16}  & ReLU   & $62.49 \pm 0.73$ & 1.08 & 7{,}512 & 5.7 \\
                           & GeLU   & $\mathbf{63.67 \pm 0.76}$ & 1.08 & 7{,}491 & 5.7 \\
                           & SiLU   & $63.60 \pm 0.26$ & 1.08 & 7{,}482 & 5.7 \\
                           & \textbf{TESLA} & $62.53 \pm 0.71$ & 1.09 & 7{,}509 & 5.7 \\
\midrule
\multirow{4}{*}{MLP-Mixer-b16} & ReLU & $57.81 \pm 0.13$ & 12.62 & 1{,}151 & 59.9 \\
                           & GeLU & $57.42 \pm 0.01$ & 12.62 & 1{,}144 & 59.9 \\
                           & SiLU & $58.36 \pm 0.91$ & 12.62 & 1{,}146 & 59.9 \\
                           & \textbf{TESLA} & $\mathbf{58.79 \pm 0.02}$ & 12.75 & 1{,}139 & 59.9 \\
\midrule
\multirow{4}{*}{ResNet-18} & ReLU  & $73.14 \pm 0.31$ & 1.82 & 8{,}447 & 11.7 \\
                           & GeLU  & $73.93 \pm 0.47$ & 1.82 & 8{,}307 & 11.7 \\
                           & SiLU  & $73.99 \pm 0.66$ & 1.82 & 8{,}326 & 11.7 \\
                           & \textbf{TESLA} & $\mathbf{74.01 \pm 0.33}$ & 1.82 & 8{,}433 & 11.7 \\
\midrule
\multirow{4}{*}{MobileNetV3-S} & ReLU  & $68.20 \pm 0.16$  & 0.05 & 35{,}607 & 2.0 \\
                           & GeLU  & $\mathbf{68.23 \pm 0.94}$ & 0.05 & 35{,}082 & 2.0 \\
                           & SiLU  & $67.92 \pm 0.15$ & 0.05 & 35{,}056 & 2.0 \\
                           & \textbf{TESLA} & $67.82 \pm 0.17$ & 0.05 & 35{,}483 & 2.0 \\
\bottomrule
\end{tabular}
}
\label{tab:act_compute_acc}
\end{table}

\subsection{Result Analysis}
\paragraph{Parity.} Figure~\ref{fig:main_acc_parity} shows the performance of TESLA and baseline models. For the parity experiments we use a fully-connected MLP with 2 hidden layers and hidden dimension $H=128$.
Across the plotted settings, TESLA matches the Bayes limit across noise levels while remaining near $100\%$ on the clean task. For $p\in\{0.05,0.10,0.20,0.30\}$, TESLA’s final accuracy concentrates near $\{95\%,90\%,80\%,70\%\}$. In contrast, common activations always collapse toward 50\% as the number of bits grows, which is no better than random guessing of parity. In the LPN setting, TESLA sustains high accuracy up to 27-bits before degrading in the hardest regimes.

\paragraph{Periodic baselines on parity.}
Table~\ref{tab:parity_noise} compares TESLA against periodic and frequency-based baselines. For easier settings (16--20 bits), TESLA and SIREN both reach 100\% accuracy under no noise. However, as the bit-length and noise increase, SIREN collapses to random guessing ($\approx 50\%$) while TESLA remains stable (e.g., at 24 bits with 30\% noise, TESLA achieves 65.53\% vs.\ 49.47\% for SIREN; at 32 bits with no noise, TESLA achieves 95.87\% vs.\ 50.10\% for SIREN). SNAKE remains near chance across all settings, and Fourier-Emb similarly fails to capture the global interaction signal in these regimes.

\begin{table}[t]
\centering
\small
\setlength{\tabcolsep}{3.5pt}
\begin{tabular}{cccccc}
\toprule
\multirow{2}{*}{Bits} & \multirow{2}{*}{Activation} & \multicolumn{4}{c}{Noise probability} \\
\cmidrule(lr){3-6}
 &  & 0.0 & 0.1 & 0.2 & 0.3 \\
\midrule
\multirow{4}{*}{16}
 & TESLA($K=8$)          & \textbf{100.00} & 89.80 & 79.22 & 68.00 \\
 & SIREN($w_0=30.0$)    & \textbf{100.00} & \textbf{89.83} & \textbf{79.61} & \textbf{69.53} \\
 & SNAKE($\alpha=1.0$)  & 49.95  & 49.97 & 49.58 & 49.89 \\
 & Fourier-Emb          & 73.21  & 50.19 & 50.09 & 50.34 \\
\midrule
\multirow{4}{*}{20}
 & TESLA($K=8$)          & \textbf{100.00} & 89.52 & 78.78 & 67.14 \\
 & SIREN($w_0=30.0$)    & \textbf{100.00} & \textbf{89.74} & \textbf{79.42} & \textbf{69.67} \\
 & SNAKE($\alpha=1.0$)  & 50.01  & 50.02 & 50.06 & 50.07 \\
 & Fourier-Emb          & 52.13  & 49.68 & 50.12 & 50.33 \\
\midrule
\multirow{4}{*}{24}
 & TESLA($K=8$)          & \textbf{100.00} & 89.75 & 78.50 & \textbf{65.53} \\
 & SIREN($w_0=30.0$)    & \textbf{100.00} & \textbf{90.08} & \textbf{80.06} & 49.47 \\
 & SNAKE($\alpha=1.0$)  & 50.24  & 49.90 & 50.52 & 49.97 \\
 & Fourier-Emb          & 50.55  & 49.20 & 50.20 & 50.00 \\
\midrule
\multirow{4}{*}{28}
 & TESLA($K=8$)          & \textbf{96.97}  & \textbf{89.61} & \textbf{78.10} & \textbf{68.14} \\
 & SIREN($w_0=30.0$)    & 50.48  & 50.52 & 49.58 & 49.78 \\
 & SNAKE($\alpha=1.0$)  & 49.95  & 49.36 & 50.02 & 50.19 \\
 & Fourier-Emb          & 49.74  & 50.45 & 50.27 & 49.81 \\
\midrule
\multirow{4}{*}{32}
 & TESLA($K=8$)          & \textbf{95.87}  & \textbf{85.78} & \textbf{78.05} & \textbf{68.78} \\
 & SIREN($w_0=30.0$)    & 50.10  & 49.89 & 50.02 & 49.95 \\
 & SNAKE($\alpha=1.0$)  & 50.24  & 50.37 & 50.34 & 50.10 \\
 & Fourier-Emb          & 50.06  & 49.42 & 50.12 & 50.48 \\
\bottomrule
\end{tabular}
\caption{Parity with long bit-length and label noise: validation accuracy(\%) (20 epochs) under matched architectures and training settings.}
\label{tab:parity_noise}
\end{table}

\paragraph{Forrelation.} We evaluate a two-layer MLP on the $n=12$ Forrelation dataset. Figure~\ref{fig:main_acc_forrelation} reports validation accuracy over epochs for hidden sizes $H=\{128,512,1024\}$. TESLA consistently learns faster and reaches higher final accuracy than standard activations across all $K$ settings. Larger hidden sizes improve performance for every method. The gap in favor of TESLA remains, which indicates that TESLA captures the global correlation signal more effectively.

% 수정 후
\paragraph{ImageNet-100.}
Table~\ref{tab:act_compute_acc} reports results on ViT-T/16, MLP-Mixer, ResNet-18, and MobileNetV3. To ensure architecture-aware comparison, we keep activations inside convolutional blocks unchanged for CNNs and replace only non-convolutional activations (e.g., MLP heads, projection/FFN layers) with TESLA; for ViT and MLP-Mixer, we replace all pointwise activations. TESLA is trained with an explicit $\ell_1$ penalty enforcing a layer-wise coefficient budget $A_K$ for stability and generalization control. Across models, TESLA remains competitive in Top-1 accuracy while keeping FLOPs and throughput within about 1\% of standard activations. These results indicate that TESLA extends beyond synthetic parity-style benchmarks as a practical drop-in replacement for standard activations in realistic vision workloads.

\paragraph{Stability vs.\ sinusoidal baselines.}
To assess whether TESLA remains numerically stable in practical architectures, we also compare against SIREN under the same ImageNet-100 training settings. In these runs, SIREN exhibits severe optimization instability and much lower accuracy, whereas TESLA trains stably and reaches standard accuracy levels (see Table~\ref{tab:siren_vs_tesla_imagenet100}).

\begin{table}[t]
\centering
\small
\setlength{\tabcolsep}{6pt}
\renewcommand{\arraystretch}{1.12}
\begin{tabular}{l l c}
\toprule
Model & Act & Top-1 (\%) \\
\midrule
\multirow{2}{*}{ViT-T/16}      & SIREN & $4.11 \pm 0.67$ \\
                                 & TESLA & $\mathbf{62.53 \pm 0.71}$ \\
\midrule
\multirow{2}{*}{MLP-Mixer-b16} & SIREN & $32.41 \pm 10.45$ \\
                                 & TESLA & $\mathbf{58.79 \pm 0.02}$ \\
\midrule
\multirow{2}{*}{ResNet-18}     & SIREN & $37.41 \pm 1.34$ \\
                                 & TESLA & $\mathbf{74.01 \pm 0.33}$ \\
\midrule
\multirow{2}{*}{MobileNetV3-S} & SIREN & $39.53 \pm 0.85$ \\
                                 & TESLA & $\mathbf{67.82 \pm 0.17}$ \\
\bottomrule
\end{tabular}
\caption{ImageNet-100 comparison between TESLA and SIREN under identical training settings. Top-1 accuracy is mean $\pm$ std.}
\label{tab:siren_vs_tesla_imagenet100}
\end{table}

\section{FUTURE WORK}
TESLA opens several promising avenues for future research. Theoretically, we plan to extend our continuous Fourier and NTK analysis to discrete Boolean domains to establish rigorous approximation bounds and explain the empirical gains observed on parity-type tasks. On the systems side, we will investigate sparse factorizations, low-bit quantization, and hardware-aware implementations to preserve TESLA’s spectral benefits while minimizing inference latency. Finally, we aim to integrate TESLA into large-scale Transformers and LLMs—such as augmenting FFN activations or positional encodings—to evaluate its impact on long-range reasoning, compositional generalization, and sample efficiency.

\section{CONCLUSION}
In this work, we introduced TESLA, a learnable sinusoidal activation that enables explicit degree/frequency control while preserving stable optimization through coefficient budgeting. We provided Lipschitz and Rademacher-complexity bounds and analyzed mode-wise learning dynamics to connect the parameterization to frequency-selective behavior.

Empirically, TESLA consistently improves tasks that require global, high-order interactions. On parity with label noise, it remains well above chance at larger bit-lengths and heavy corruption, and is more stable than periodic/frequency-based baselines (SIREN, SNAKE, and Fourier feature embeddings) in the hardest regimes. On Forrelation, TESLA achieves the highest accuracy across widths, while periodic and Fourier-style baselines remain near chance. On continuous-domain tasks (PINNs, INRs, and mixed-frequency regression), TESLA is competitive with or better than both standard and sinusoidal activations. On ImageNet-100, TESLA remains competitive across architectures with modest overhead and is substantially more stable than SIREN. Overall, TESLA suggests that activation-level degree control is a robust, practical alternative to input-only spectralization.

\newpage
\bibliographystyle{apalike}
\bibliography{reference}

\clearpage
\setcounter{section}{0}
\setcounter{subsection}{0}
\setcounter{figure}{0}
\setcounter{table}{0}
\setcounter{equation}{0}
\renewcommand{\theHsection}{supp.\arabic{section}}
\renewcommand{\theHsubsection}{supp.\arabic{section}.\arabic{subsection}}
\renewcommand{\theHfigure}{supp.\arabic{figure}}
\renewcommand{\theHtable}{supp.\arabic{table}}
\renewcommand{\theHequation}{supp.\arabic{equation}}
\fancyhead[CE,CO]{\small\bfseries TESLA: Sinusoidal Learnable Activations}
\input{appendix_body}

\end{document}

%% file: appendix_body.tex
\runningtitle{TESLA: Sinusoidal Learnable Activations}

\runningauthor{Daehwa Ko, Jaehyeon Kim, Seunghyun Ham, Jay Hoon Jung}

% AISTATS supplementary material must use a single-column layout.
\onecolumn
\aistatstitle{TESLA: Taylor Expansion of Sinusoidal Learnable Activations \\
Supplementary Materials}
\thispagestyle{empty}

\section{Experiment environment summary}
\label{appendix:repro}

This appendix provides details, such as hyperparameter, computer hardware and software environment information, necessary to reproduce the experiments reported in the main text. $K$ denotes the number of sinusoidal harmonics in each TESLA layer, and $A_K$ is the $l_1$-style coefficient budget for controlling their amplitude.

\begin{table*}[h]
  \centering
  \caption{Hyperparameter summary for main experiments. Each reported experiment in the paper uses the settings below unless noted. All task were conducted on backbone with MLP, 2 layers, and optimizer with Adam and $A_k = 0$.}
  \label{tab:hyperparams}
  \begin{tabular}{lcccccc}
    \toprule
    Task & Width & $K$ & LR & Batch Size & Epochs\\
    \midrule
    Parity ($d=16-32$, $n=100k$) & 128 & d/4 & 1e-3 & 1024 & 30 \\
    Forrelation ($d=12$) & [128, 512, 1024] & [1, 2, 4, 16] & 1e-3 & 128 & 20 \\
    LPN (noise $=0.1$) & 128 & d/4 & 5e-4 & 1024 & 30  \\
    \bottomrule
  \end{tabular}
\end{table*}

\begin{table*}[h]
  \centering
  \caption{Hyperparameter summary for ImageNet-100. All models are trained with AdamW for 50 epochs using cosine LR decay. Warmup uses a linear schedule; Warmup ratio is the LR multiplier at the first warmup step. All backbones fix the batch size to 128, set the learning rate to 1e-3, perform warm-up for 5 epochs, and set the warm-up ratio to 0.85.}
  \label{tab:hyperparams_image}
  \setlength{\tabcolsep}{8pt}
  \begin{tabular}{lcccccc}
    \toprule
    Backbone & $K$ & $A_{K}$ \\
    \midrule
    ViT-T/16       & 2 & 1e-3 \\
    MLP-Mixer-b16  & 6 & 1e-1 \\
    ResNet-18      & 2 & 1e-3 \\
    MobileNetV3-S  & 2 & 1e-3 \\
    \bottomrule
  \end{tabular}
\end{table*}

\begin{table}[h]
  \centering
  \caption{Compute and software environment used for all experiments unless otherwise noted.}
  \label{tab:env}
  \begin{tabular}{@{}p{0.30\linewidth}p{0.66\linewidth}@{}}
    \toprule
    \textbf{Hardware} & NVIDIA A100 80GB (1x), AMD EPYC 7543 32-Core Processor, 256GB RAM \\
    \textbf{OS} & Ubuntu 22.04 LTS \\
    \textbf{Python / Framework} & Python 3.10, PyTorch 2.0+ \\
    \textbf{CUDA / cuDNN / Driver} & CUDA 11.8, cuDNN 8.x, NVIDIA Driver 535+ \\
    \textbf{Key Python deps} & numpy 1.26, scipy 1.14, triton 3.1.0, tqdm 4.x \\
    \textbf{Reproducibility} & Fixed seeds $\{0,1,2\}$; all results reported as mean $\pm$ std over seeds. \\
    \bottomrule
  \end{tabular}
\end{table}

\section{Representations and Generalization on Toy Tasks: Synthetic Data \& Parity}
This section complements the theory by visualizing how TESLA differs from standard activations on synthetic 2D tasks and on the 10-bit parity mapping.

% \paragraph{Observations.}
% (\textit{i}) On Spiral/Checkerboard, TESLA yields smooth, globally consistent boundaries and retains oscillatory structure, while ReLU/KAN exhibit piecewise distortions or aliasing.
% (\textit{ii}) On parity, TESLA with modest $K$ already recovers the high-frequency checker pattern, in line with the mode-wise spectrum of the TESLA NTK (Section~\ref{sec:modewise}); increasing $K$ sharpens the grid without introducing spurious transitions.
% (\textit{iii}) These qualitative trends match the generalization story in Section~\ref{sec:modewise}: medium--high modes receive relatively larger eigenvalues under TESLA, accelerating recovery of higher-order structure.

\begin{figure*}[t]
  \centering
  \begin{subfigure}[b]{0.48\textwidth}
    \centering
    \includegraphics[width=\linewidth]{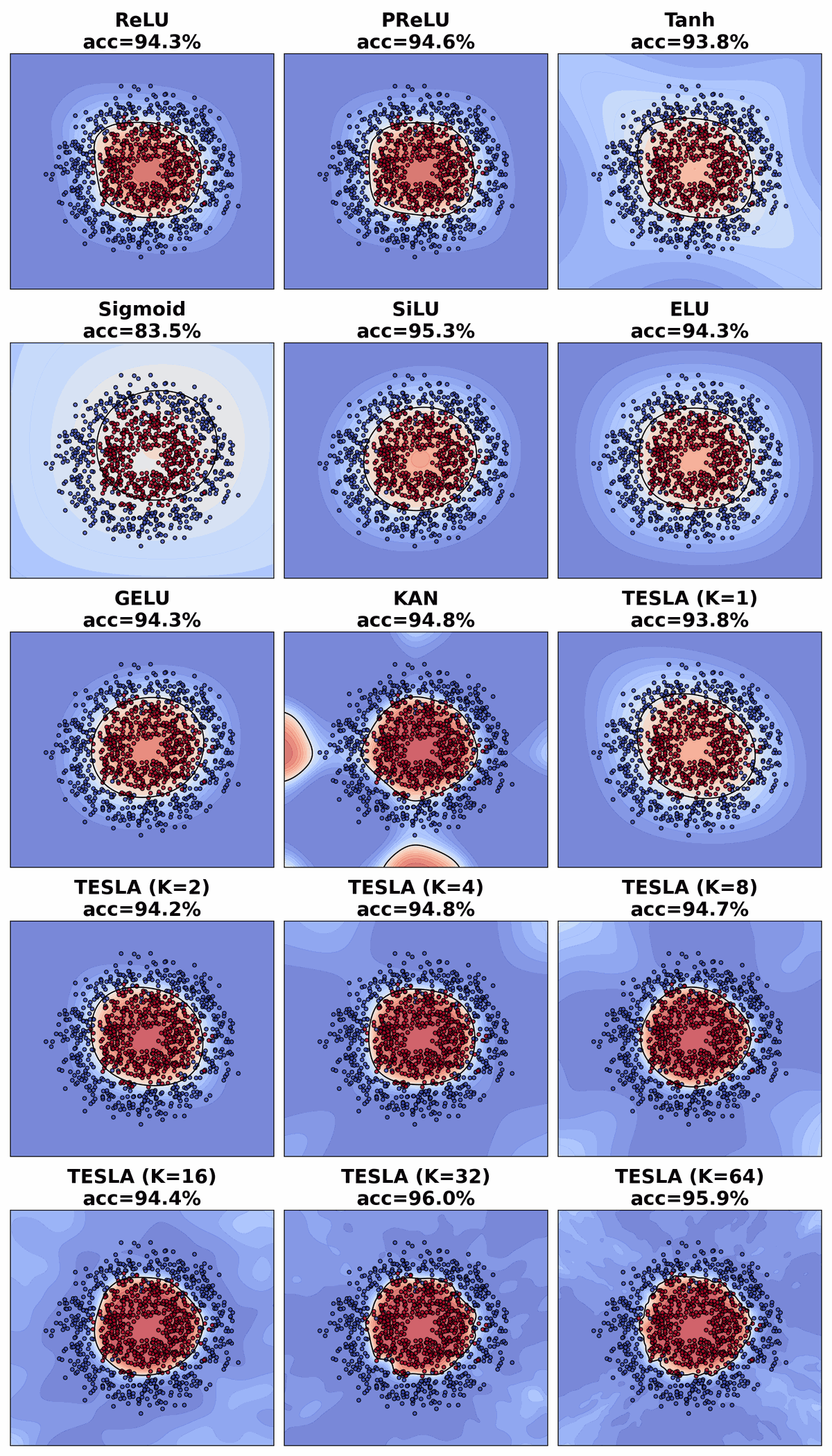}
    \caption{\textbf{Concentric Circles}.}
    \label{fig:dec_circles}
  \end{subfigure}
  \hfill
  \begin{subfigure}[b]{0.48\textwidth}
    \centering
    \includegraphics[width=\linewidth]{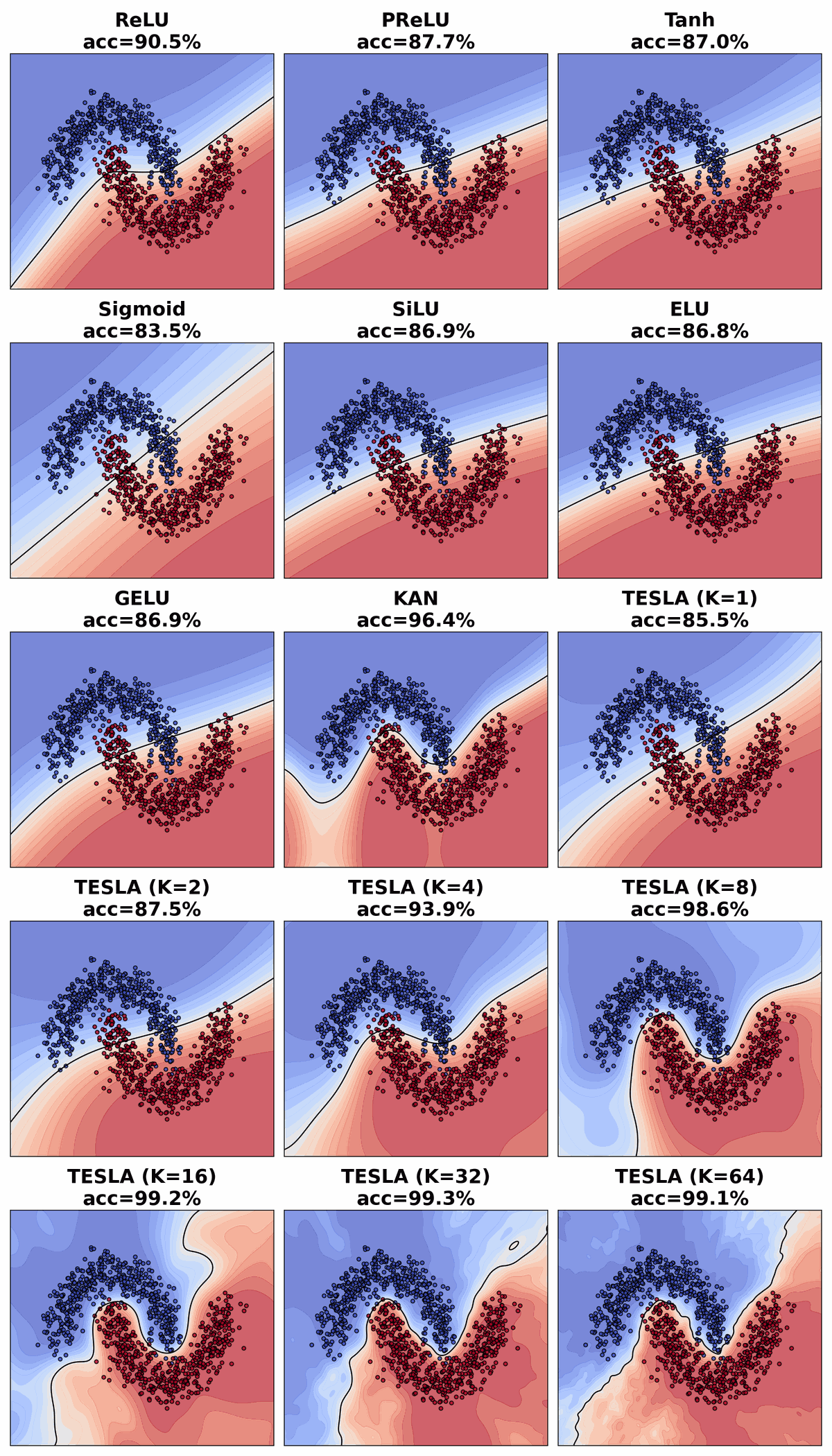}
    \caption{\textbf{Moons}.}
    \label{fig:dec_moons}
  \end{subfigure}
  \caption{Decision-boundary comparison on two synthetic datasets. A two-layer MLP with 128 hidden units is trained for 200 epochs on 1,000 samples. TESLA is evaluated with multiple $K$ and compared against KAN and standard activations. Subpanel titles report training accuracy in percent. All subpanels share the same color and legend scale.}
  \label{fig:decision_boundaries_set1}
\end{figure*}

\begin{figure*}[t]
  \centering
  \begin{subfigure}[b]{0.48\textwidth}
    \centering
    \includegraphics[width=\linewidth]{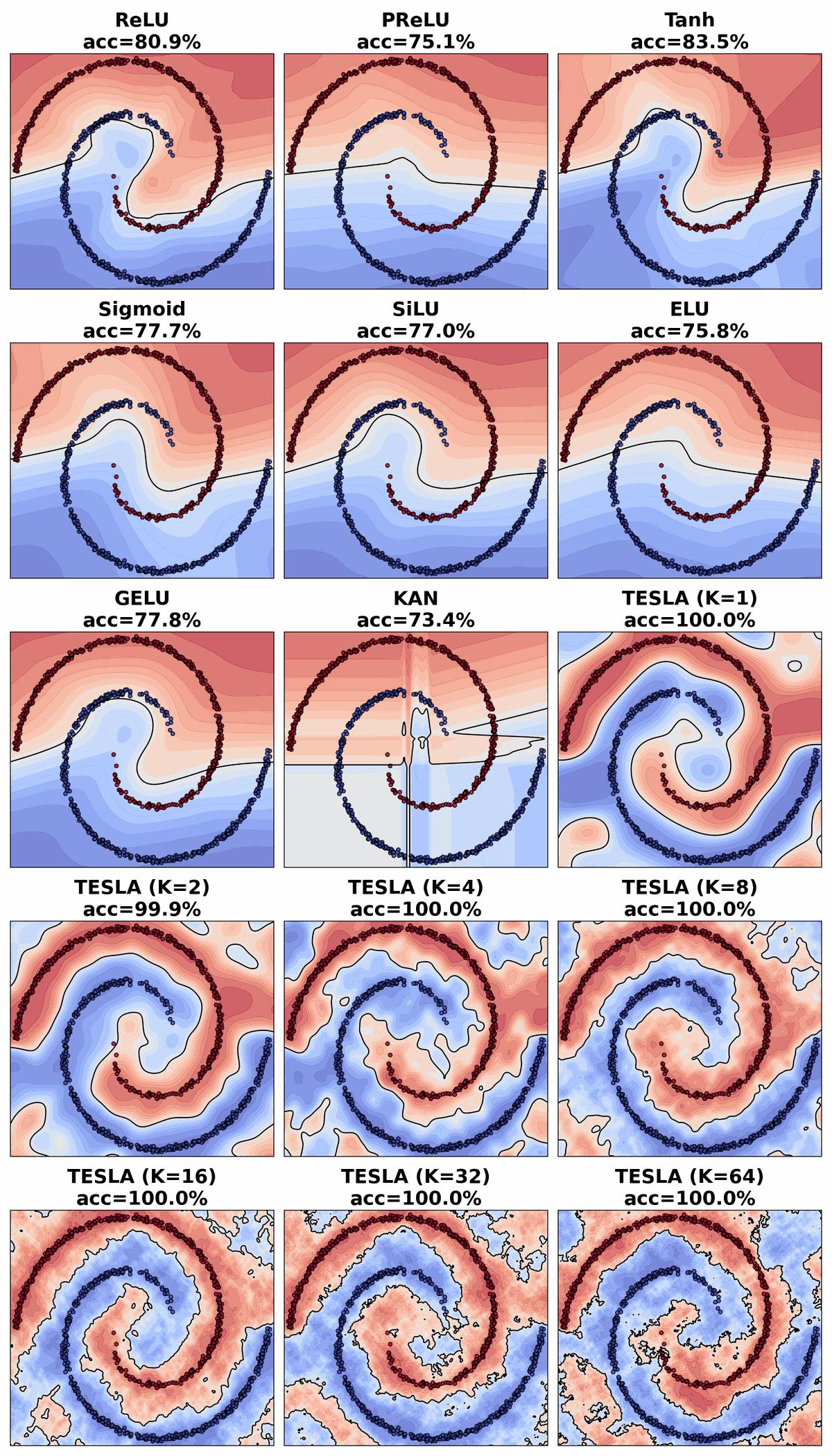}
    \caption{\textbf{Spiral}.}
    \label{fig:dec_spiral}
  \end{subfigure}
  \hfill
  \begin{subfigure}[b]{0.48\textwidth}
    \centering
    \includegraphics[width=\linewidth]{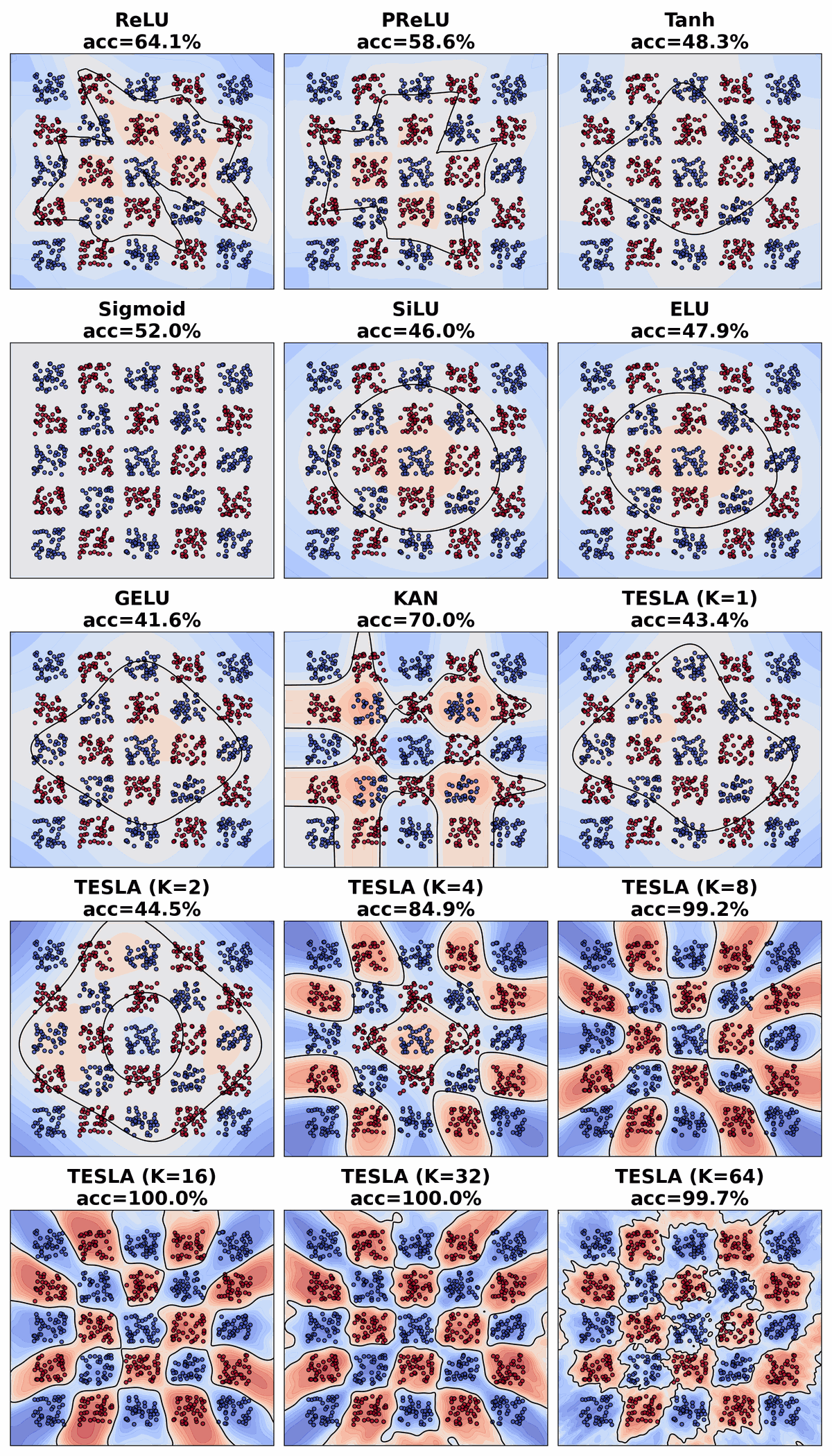}
    \caption{\textbf{Checkerboard}.}
    \label{fig:dec_checkerboard}
  \end{subfigure}
  \caption{Decision-boundary comparison on two additional datasets. Training setup matches Figure \ref{fig:decision_boundaries_set1}. TESLA consistently recovers global or oscillatory structure, whereas baselines often miss these patterns. Subpanel titles report training accuracy in percent. Color and legend scales are shared across subpanels.}
  \label{fig:decision_boundaries_set2}
\end{figure*}

\begin{figure*}[t]
    \centering
    \includegraphics[width=0.90\textwidth]{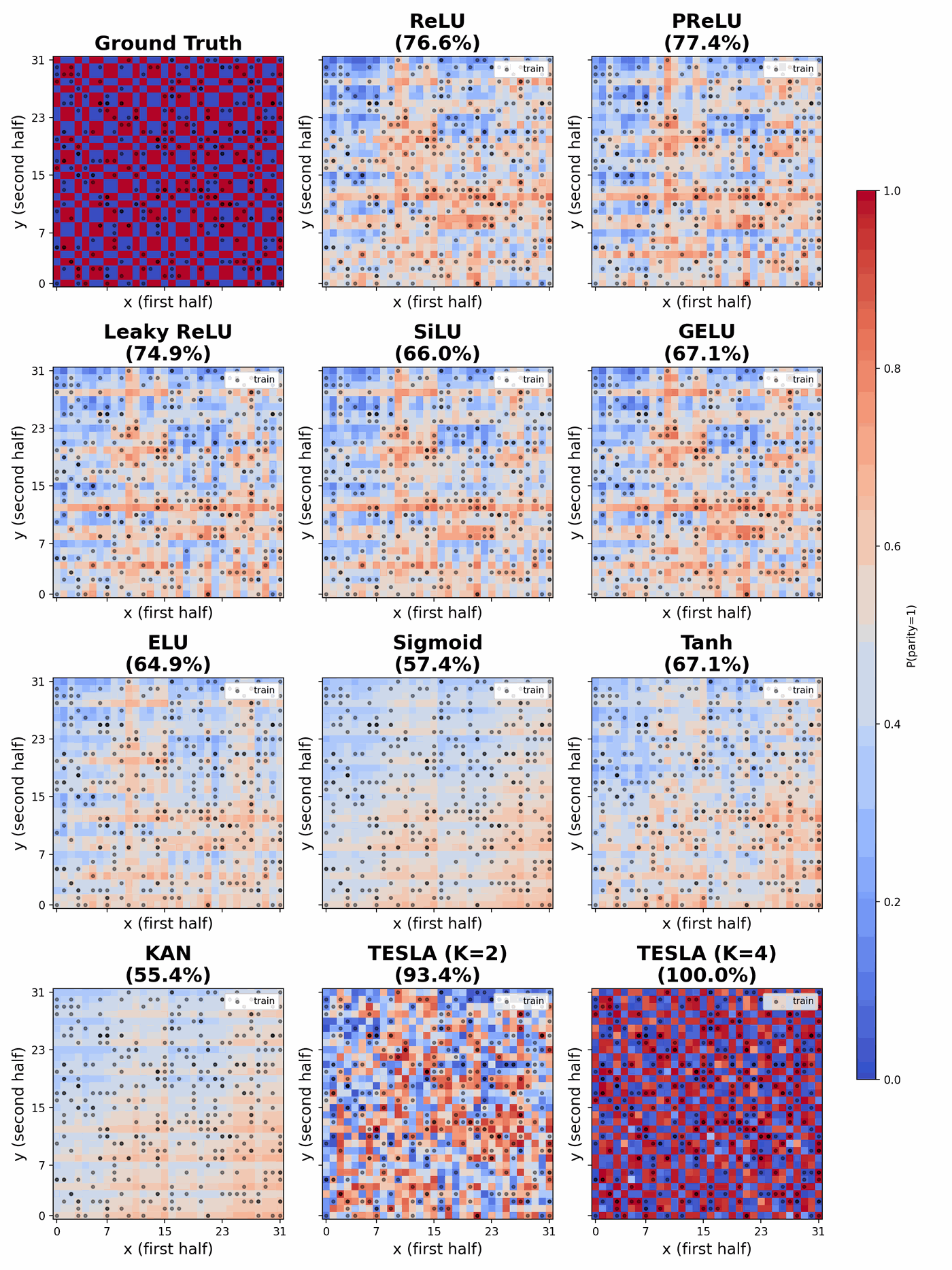}
    \caption{\textbf{Parity decision surfaces on 10-bit inputs.} A two-layer MLP (hidden = 128, lr = 1e-3, 50 epochs) is trained on 350 samples with different activations. Each 10-bit string is split into two 5-bit halves mapped to the x- and y-axes as binary integers (0–31), yielding a 32×32 grid. Color shows predicted $P(y=1)$ (blue = 0, red = 1). Black dots mark training points. The top-left panel is ground truth. Panel titles report training accuracy (\%).}
    \label{fig:parity_boundaries}
\end{figure*}

% \clearpage
% \bibliographystyle{abbrvnat}
% \bibliography{reference}